\documentclass{article}

\PassOptionsToPackage{round,compress}{natbib}
\PassOptionsToPackage{colorlinks,allcolors=blue,backref=page}{hyperref}

\usepackage[frak=boondox]{mathalfa}

\usepackage[
    preprint, 
]{neurips_2023}

\usepackage[utf8]{inputenc} 
\usepackage[T1]{fontenc}    
\usepackage{hyperref}       
\usepackage{url}            
\usepackage{booktabs}       
\usepackage{amsfonts}       
\usepackage{nicefrac}       
\usepackage{microtype}      
\usepackage{xcolor}         

\newcommand{\anon}[2][]{#2}

\usepackage{graphicx}
\usepackage{enumitem}
\usepackage{tikz}

\usepackage{amsmath}

\usepackage[capitalize,noabbrev]{cleveref}

\renewcommand*{\backref}[1]{}
\renewcommand*{\backrefalt}[4]{%
    \ifcase #1\or Cited on page~#2.\else Cited on pages~#2.\fi%
}

\colorlet{accent1}{gray}
\colorlet{accent2}{red}

\usepackage{mathtools}
\usepackage{amssymb}

\newcommand\Reals{{\mathbb{R}}}

\newcommand\Natz{{\mathbb{N}}}
\newcommand\PosNats{{\mathbb{N}^{+}}}

\newcommand\set[1]{\left\{#1\right\}}
\newcommand\Set[2]{{\left\{\,#1\,\middle|\,#2\,\right\}}}
\newcommand\sign[1]{\operatorname{sign}\!\left({#1}\right)}

\newcommand\cardinality[1]{{\left|#1\right|}}
\newcommand\range[2]{{{#1},\ldots,{#2}}}
\newcommand\setrange[2]{{\set{{#1},\ldots,{#2}}}}
\newcommand\subrange[3]{{#3_{#1},\ldots,#3_{#2}}}
\newcommand\vecrange[3]{{(\subrange{#1}{#2}{#3})}}

\newcommand\xlt{\prec}
\newcommand\xleq{\preceq}
\newcommand\xgt{\succ}
\newcommand\xgeq{\succeq}

\newcommand\xsign[1]{\operatorname{sign_{lex}}\!\left({#1}\right)}

\newcommand\W[1][]{\mathcal{W}^{#1}}
\newcommand\prange[4][h]{{\range{#2_1,#3_1,#4_1}{#2_{#1},#3_{#1},#4_{#1}}}}
\newcommand\abcrange[1][h]{\prange[#1]{a}{b}{c}}
\newcommand\w[1]{w^{#1}}
\newcommand\pa[1]{a^{#1}}
\newcommand\pb[1]{b^{#1}}
\newcommand\pc[1]{c^{#1}}
\newcommand\pd[1]{d^{#1}}
\newcommand\rank{\operatorname{rank}}
\newcommand\fneqclass[1]{{\mathfrak{F}\!\left[{#1}\right]}}
\newcommand\CS[2][h]{\Gamma(#1, #2)}
\newcommand\CALL[1]{\textsc{#1}}

\newcommand\Path{\rho}
\newcommand\pcon{\leftrightsquigarrow}

\usepackage{amsthm} 
\usepackage{thmtools}
\usepackage{thm-restate}

\declaretheorem[parent=section]{theorem}
\declaretheorem[sibling=theorem]{lemma}
\declaretheorem[sibling=theorem,style=definition]{definition}
\declaretheorem[sibling=theorem,style=definition]{remark}
\declaretheorem[sibling=theorem,style=definition]{algorithm}

\usepackage[%
    noEnd=false,
    indLines=false,
    commentColor=accent1,
]{algpseudocodex} 

\title{%
    Functional Equivalence and Path Connectivity\\
    of Reducible Hyperbolic Tangent Networks%
}

\author{%
  Matthew Farrugia-Roberts \\
  School of Computing and Information Systems \\
  The University of Melbourne \\
  \texttt{matthew@far.in.net}
}

\begin{document}

\maketitle

\begin{abstract}
    Understanding the learning process of artificial neural networks requires
    clarifying the structure of the parameter space within which learning
    takes place.
    A neural network parameter's \emph{functional equivalence class} is the
    set of parameters implementing the same input--output function.
    For many architectures, almost all parameters have a simple and
    well-documented functional equivalence class. However, there is
    also a vanishing minority of \emph{reducible} parameters, with
    richer functional equivalence classes caused by redundancies among the
    network's units.
    
    In this paper, we give an algorithmic characterisation of unit
    redundancies and reducible functional equivalence classes for
    a single-hidden-layer hyperbolic tangent architecture.
    We show that such functional equivalence classes are piecewise-linear
    path-connected sets, and that for parameters with a majority of redundant
    units, the sets have a diameter of at most~7 linear segments.
\end{abstract}


\section{Introduction}
\label{sec:intro}

Deep learning algorithms construct a parameter for an artificial neural
network architecture through a local search in the high-dimensional
parameter space.
This search is guided by the topography of some loss landscape. This
topography is in turn determined by the relationship between neural
network parameters and neural network input--output functions.
Thus, understanding the relationship between these parameters and functions
is key to understanding deep learning.

It is well known that neural network parameters often fail to uniquely
determine an input--output function.
For example, exchanging weights between two adjacent hidden units generally
preserves functional equivalence
    \citep{Hecht-Nielsen1990}.
For many architectures, almost all parameters have a simple 
    class of functionally equivalent parameters.
These classes have been characterised for multi-layer feed-forward
architectures with various nonlinearities
    \citep[e.g.,][]{Sussmann1992,Albertini+1993,Kurkova+Kainen1994,
    Phuong+Lampert2020,Vlacic+Bolcskei2021}.

However, all existing work on functional equivalence excludes from
consideration certain measure zero sets of parameters, for which the
functional equivalence classes may be richer.
One such family of parameters is the so-called \emph{reducible parameters}.
These parameters display certain structural redundancies, such that the same
function could be implemented with fewer hidden units
    \citep{Sussmann1992,Vlacic+Bolcskei2021},
leading to a richer functional equivalence class.

Despite their atypicality, reducible parameters may play an important role in
deep learning.
Learning exerts a non-random selection pressure on parameters, and reducible
parameters are appealing solutions due to parsimony%
    \anon[.]{ \citep[cf.][]{me-prank}.}
These parameters are a source of information singularities
    \citep[cf.][]{Fukumizu1996},
relevant to statistical theories of deep learning
    \citep{greybook,goodpaper}.
Moreover, the structure of functional equivalence classes has
implications for the topography of the loss landscape, and, therefore,
for the dynamics of learning.

In this paper, we study functional equivalence classes for single-hidden-layer
networks with the hyperbolic tangent nonlinearity, building on the
foundational work of \citet{Sussmann1992} on reducibility in this setting.
While this architecture is not immediately relevant to modern deep learning,
structural redundancy has unresearched implications for functional equivalence
in \emph{all} architectures.
A comprehensive investigation of this simple case is a first step in this
research direction.
To this end, we offer the following theoretical contributions.%
\anon{\footnote{%
    Contributions
        (\ref{itm:contrib:1}), (\ref{itm:contrib:2}), and
        (\ref{itm:contrib:3})
    also appear in the author's minor thesis
        \citep[\S5]{me-thesis}.
}}
\begin{enumerate}
    \item\label{itm:contrib:1}
        In \cref{sec:canon},
        we give a formal algorithm producing a canonical representative
        parameter from any functional equivalence class,
        by systematically eliminating all sources of structural redundancy.
        This extends prior algorithms that only handle irreducible parameters.
        %

    \item\label{itm:contrib:2}
        In \cref{sec:fneqclass}, we invert this canonicalisation algorithm to
        characterise the functional equivalence class of any parameter as a
        union of simple parameter manifolds.
        This characterisation extends the well-known result for irreducible
        parameters.

    \item\label{itm:contrib:3}
        We show that in the reducible case, the functional equivalence class
        is a piecewise-linear path-connected set---that is, any two
        functionally equivalent reducible parameters are connected by a
        piecewise linear path comprising only equivalent parameters
            (\cref{thm:connectivity}).

    \item\label{itm:contrib:4}
        We show that if a parameter has a high degree of reducibility
            (in particular, if the same function can be implemented
            using half of the available hidden units),
        then the number of linear segments required to connect any two
        equivalent parameters is at most 7
            (\cref{thm:diameter}).
\end{enumerate}

In \cref{sec:discussion}, we discuss the implications of these results for an
understanding of the structure of the parameter space, and outline directions
for future work including extensions to modern architectures.

\section{Related Work}
\label{sec:related}

\citet{Sussmann1992} studied functional equivalence in single-hidden-layer
hyperbolic tangent networks, showing that two irreducible parameters
are functionally equivalent if and only if they are related by simple
operations of exchanging and negating the weights of hidden units.
This result was later extended to architectures with a broader class of
nonlinearities
    \citep{Albertini+1993,Kurkova+Kainen1994},
to architectures with multiple hidden layers
    \citep{Fefferman+Markel1993,Fefferman1994},
and to certain recurrent architectures
\citep{Albertini+Sontag1992,Albertini+Sontag1993a,Albertini+Sontag1993b,
    Albertini+Sontag1993c}.
More recently, similar results have been found for ReLU networks
    \citep{Phuong+Lampert2020,Bona-Pellissier+2021,Stock+Gribonval2022},
and \citet{Vlacic+Bolcskei2021,Vlacic+Bolcskei2022} have generalised
Sussmann's results to a very general class of architectures and
nonlinearities.
However, all of these results have come at the expense of excluding from
consideration certain measure zero subsets of parameters with richer functional
equivalence classes.

A similar line of work has documented the global symmetries of the parameter
space---bulk transformations of the entire parameter space that preserve all
implemented functions.
The search for such symmetries was launched by \cite{Hecht-Nielsen1990}.
\citet[also \citealp{Chen+Hecht-Nielsen1991}]{Chen+1993} showed that in the
case of multi-layer hyperbolic tangent networks, all analytic symmetries
are generated by unit exchanges and negations.
\citet{Ruger+Ossen1997} extended this result to additional sigmoidal
nonlinearities.
The analyticity condition excludes discontinuous symmetries acting selectively
on, say, reducible parameters with richer equivalence classes
    \citep{Chen+1993}.

\citet{Ruger+Ossen1997} provide a canonicalisation algorithm. Their algorithm
negates each hidden unit's weights until the bias is positive, and then sorts
each hidden layer's units into non-descending order by bias weight.
This algorithm is invariant precisely to the exchanges and negations
mentioned above, but fails to properly canonicalise equivalent parameters
that differ in more complex ways.

To our knowledge there is one line of work bearing directly on the
topic of the functional equivalence classes of reducible parameters.
\citet{Fukumizu+Amari2000} and \citet{Fukumizu+2019} have catalogued methods
of adding a single hidden unit to a neural network while preserving the
network's function, and \citet{Simsek+2021} have extended this work to
consider the addition of multiple hidden units.
Though derived under a distinct framing, it turns out that the subsets of
parameter space accessible by such unit additions correspond to functional
equivalence classes, similar to those we study (though in a slightly
different architecture).
We note these similarities, especially regarding our contributions
    (\ref{itm:contrib:2}) and (\ref{itm:contrib:3}),
in
    \cref{remark:fukumizu-fneq,remark:simsek-fneq}
    and
    \cref{remark:simsek-connectivity}.

\section{Preliminaries}
\label{sec:prelims}

We consider a family of fully-connected, feed-forward neural network
architectures with
    a single input unit,
    a single biased output unit,
and
    a single hidden layer of $h \in \Natz$ biased hidden units
    with the hyperbolic tangent nonlinearity
        $\tanh(z) = (e^{z} - e^{-z}) / (e^{z} + e^{-z})$.
Such an architecture has a parameter space
    $\W_h = \Reals^{3h+1}$.
Our results generalise directly to networks with multi-dimensional
inputs and outputs, as detailed in \cref{apx:multidim}.

The weights and biases of the network's units are encoded in the parameter
vector in the format
    $(\abcrange, d) \in \W_h$
where
    for each hidden unit $i = \range1h$ there is
        an \emph{outgoing weight} $a_i \in \Reals$,
        an \emph{incoming weight} $b_i \in \Reals$,
        and a \emph{bias} $c_i \in \Reals$,
    and $d \in \Reals$ is
        an \emph{output unit bias}.
Thus each parameter $w = (\abcrange, d) \in \W_h$ indexes a
mathematical function $f_w : \Reals \to \Reals$ defined as follows:
\begin{equation*}
    f_w(x) = d + \sum_{i=1}^h a_i \tanh(b_i x + c_i).
\end{equation*}

Two parameters $w \in \W_h, w' \in \W_{h'}$ are \emph{functionally
equivalent} if and only if
    $f_w = f_{w'}$ as functions on $\Reals$
(that is, $\forall x \in \Reals, f_w(x) = f_{w'}(x)$).
Functional equivalence is of course an equivalence relation on $\W_h$.
Given a parameter $w \in \W_h$, the \emph{functional equivalence class} of
$w$, denoted $\fneqclass{w}$, is the set of all parameters in $\W_h$ that
are functionally equivalent to $w$:
\begin{equation*}
    \fneqclass{w} = \Set{w' \in \W_h}{f_w = f_{w'}}.
\end{equation*}

For this family of architectures, the functional equivalence class of almost
all parameters is a discrete set fully characterised by simple \emph{unit
negation and exchange transformations}
    $\sigma_i, \tau_{i,j} : \W_h \to \W_h$
for
    $i,j=\range1h$, where
\begin{align*}
    \sigma_i(\abcrange, d)
    &= (
        a_1,b_1,c_1,
        \ldots,
        {\color{accent2}-}a_i,
        {\color{accent2}-}b_i,
        {\color{accent2}-}c_i,
        \ldots,
        a_h,b_h,c_h,
        d
    )
\\ \nonumber
    \tau_{i,j}(\abcrange, d)
    &= (
        a_1, b_1, c_1,
        \ldots,
        c_{i-1},
        a_{\color{accent2}j},
        b_{\color{accent2}j},
        c_{\color{accent2}j},
        a_{i+1},
    \\ & \hspace{5.6em} 
        \ldots,
        c_{j-1},
        a_{\color{accent2}i},
        b_{\color{accent2}i},
        c_{\color{accent2}i},
        a_{j+1},
        \ldots,
        a_h,b_h,c_h,
        d
    ).
\end{align*}
More formally, these transformations generate the full functional equivalence
class for all so-called irreducible parameters
    \citep{Sussmann1992}.
A parameter $w = (\abcrange, d) \in \W_h$ is \emph{reducible} if and only if
it satisfies any of the following conditions
    (otherwise, $w$ is \emph{irreducible}):
\begin{enumerate}[label=(\roman*)]
    \item\label{item:reducibility:1}
        $a_i = 0$ for some $i$, or
    \item\label{item:reducibility:2}
        $b_i = 0$ for some $i$, or
    \item\label{item:reducibility:3}
        $(b_i, c_i) = (b_j, c_j)$ for some $i \neq j$, or
    \item\label{item:reducibility:4}
        $(b_i, c_i) = (-b_j, -c_j)$ for some $i \neq j$.
\end{enumerate}

\citet{Sussmann1992} also showed that in this family of architectures,
reducibility corresponds to \emph{non-minimality}:
a parameter $w \in \W_h$ is reducible if and only if $w$ is functionally
equivalent to some $w' \in \W_{h'}$ with fewer hidden units $h' < h$.
We define the \emph{rank} of $w$, denoted $\rank(w)$, as the minimal number
of hidden units required to implement $f_w$:
\begin{equation*}
    \rank(w) = \min\Set{h' \in \Natz}{\exists w' \in \W_{h'};\ f_w = f_{w'}}.
\end{equation*}

Finally, we make use of the following notions of connectivity for a set of
parameters.
Given a set $W \subseteq \W_h$, define a \emph{piecewise linear path in $W$}
as a continuous function $\Path : [0,1] \to W$ comprising a finite number of
linear segments.
Two parameters $w, w' \in \W_h$ are \emph{piecewise-linear path-connected in
$W$}, denoted $w \pcon w'$ (with $W$ implicit), if there exists a piecewise
linear path in $W$ such that $\Path(0) = w$ and $\Path(1) = w'$.
Note that $\pcon$ is an equivalence relation on $W$.
A set $W \subseteq \W_h$ is itself \emph{piecewise-linear path-connected} if
and only if $\pcon$ is full, that is, all pairs of parameters in $W$ are
piecewise linear path-connected in $W$.

The \emph{length} of a piecewise linear path is the number of maximal
linear segments comprising the path.
The \emph{distance} between two piecewise linear path-connected parameters is
the length of the shortest path connecting them.
The \emph{diameter} of a piecewise linear path-connected set is the largest
distance between any two parameters in the set.

\section{Parameter Canonicalisation}
\label{sec:canon}

A parameter \emph{canonicalisation algorithm} maps each parameter in a
functional equivalence class to a canonical representative parameter within
that class.
A canonicalisation algorithm therefore serves as a computational test of
functional equivalence.

Prior work has described canonicalisation algorithms for certain irreducible
parameters \citep{Ruger+Ossen1997}; but when applied to functionally equivalent
reducible parameters, such algorithms may fail to produce the same output.
We introduce a canonicalisation algorithm that properly canonicalises both
reducible and irreducible parameters, based on similar negation and sorting
stages, combined with a novel \emph{reduction} stage. This stage effectively
removes or `zeroes out' redundant units through various operations,
isolating a functionally equivalent but irreducible subparameter.

\begin{algorithm}[Parameter canonicalisation]
    \label{algo:parameter-canonicalisation}
    Given a parameter space $\W_h$, proceed:
    \begin{algorithmic}[1]
        \Procedure{Canonicalise}{$w=(\abcrange, d) \in \W_h$}
            \LComment{%
                Stage 1: Reduce the parameter, zeroing out redundant hidden
                units
            }
            \State $Z \gets \set{}$ \Comment{keep track of `zeroed' units}
            \While{any of the following four conditions hold}
                \If{for some hidden unit $i \notin Z$, $a_i = 0$}
                    \Comment{reducibility condition \ref{item:reducibility:1}}
                    \State $b_i, c_i \gets 0$
                    \State $Z \gets Z \cup \set{i}$
                \ElsIf{for some hidden unit $i \notin Z$, $b_i = 0$}
                    \Comment{------ \ref{item:reducibility:2}}
                    \State $d \gets d + a_i \tanh(c_i)$
                    \State $a_i, c_i \gets 0$
                    \State $Z \gets Z \cup \set{i}$
                \ElsIf{
                    for some hidden units $i, j \notin Z, i \neq j$,
                    $(b_i,c_i) = (b_j, c_j)$
                }
                    \Comment{------ \ref{item:reducibility:3}}
                    \State $a_j \gets a_j + a_i$
                    \State $a_i, b_i, c_i \gets 0$
                    \State $Z \gets Z \cup \set{i}$
                \ElsIf{
                    for some hidden units $i, j \notin Z, i \neq j$,
                    $(b_i,c_i) = (-b_j, -c_j)$
                }
                    \Comment{------ \ref{item:reducibility:4}}
                    \State $a_j \gets a_j - a_i$
                    \State $a_i, b_i, c_i \gets 0$
                    \State $Z \gets Z \cup \set{i}$
                \EndIf
            \EndWhile

            \LComment{
                Stage 2: Negate the nonzero units to have positive incoming
                weights
            }
            \For{each hidden unit $i \notin Z$}
                \State $a_i,b_i,c_i \gets \sign{b_i} \cdot (a_i,b_i,c_i)$
            \EndFor

            \LComment{
                Stage 3: Sort the units by their incoming weights and biases
            }
            \State
                $\pi \gets$ a permutation sorting $i=\range1h$
                    by decreasing $b_i$, breaking ties with decreasing $c_i$
            \State
                $
                    w \gets (
                        a_{\pi(1)}, b_{\pi(1)}, c_{\pi(1)},
                        \ldots,
                        a_{\pi(h)}, b_{\pi(h)}, c_{\pi(h)},
                        d
                    )
                $
            \LComment{
                Now, $w$ has been mutated into the canonical equivalent
                parameter
            }
            \State \Return $w$
        \EndProcedure
    \end{algorithmic}
\end{algorithm}

The following theorem establishes the correctness of
    \cref{algo:parameter-canonicalisation}.

\begin{theorem}
    \label{thm:parameter-canonicalisation}
    Let $w, w' \in \W_h$.
    Let $v = \CALL{Canonicalise}(w)$ and $v' = \CALL{Canonicalise}(w')$.
    Then
    \begin{enumerate}[label=(\roman*)]
        \item $v$ is functionally equivalent to $w$; and
        \item if $w$ and $w'$ are functionally equivalent, then $v = v'$.
    \end{enumerate}
\end{theorem}

\begin{proof}
    For (i), observe that $f_w$ is maintained by each iteration of the loops
        in Stages 1 and 2, and by the permutation in Stage 3.
    For (ii), observe that Stage 1 isolates
        functionally equivalent \emph{and irreducible} subparameters
        $u \in \W_r$ and $u' \in \W_{r'}$
    of the input parameters $w$ and $w'$
        (excluding the zeroed units).
    We have $f_u = f_w = f_{w'} = f_{u'}$, so by the results of
        \citet{Sussmann1992},
    $r = r' = \rank(w)$, and $u$ and $u'$ are related by unit negation and
    exchange transformations.
    This remains true in the presence of the zero units.
    Stages 2 and 3 are invariant to precisely such transformations by
    construction.
\end{proof}

\section{Full Functional Equivalence Class}
\label{sec:fneqclass}

\Cref{algo:parameter-canonicalisation} produces a consistent output for all
parameters within a given functional equivalence class. It serves as the
basis for the following characterisation of the full functional equivalence
class.

The idea behind the characterisation is to enumerate the various ways for a
parameter's units to be reduced, negated, and sorted throughout
    \cref{algo:parameter-canonicalisation}.
Each such \emph{canonicalisation trace} corresponds to a simple set of
parameters that takes exactly this path through the algorithm, as follows.

\begin{definition}[Canonicalisation trace]
    \label{def:canonicalisation-trace}
    Let $r, h \in \Natz$, $r \leq h$. A
        \emph{canonicalisation trace of order $r$ on $h$ units}
    is a tuple
        $(\sigma, \tau)$,
    where $\sigma \in \set{-1,+1}^h$ is a sign vector
        (interpreted as tracking unit negation throughout the algorithm);
    and $\tau : \setrange1h \to \set{0,\range1h}$ is a function with range
        including $\setrange1r$
        (interpreted as tracking unit reduction and permutation 
        throughout the algorithm).
\end{definition}

\begin{theorem}
    \label{thm:characterisation}
    Let
        $w \in \W_h$
    and
        $v = (\prange\alpha\beta\gamma,\delta) = \CALL{Canonicalise}(w)$.
    Let
        $r = \rank(w)$.
    Then the functional equivalence class $\fneqclass{w} \subset \W_h$ is a
    union of subsets
    \begin{equation}
        \label{eq:characterisation}
        \fneqclass{w}
        =
        \bigcup_{(\sigma, \tau) \in \CS{r}}
        \left(
            X^\delta_{\tau^{-1}[0]}
            \cap
                \smashoperator{\bigcap_{i=1}^r}
                Y^{\alpha_i,\beta_i,\gamma_i}_{\sigma,\tau^{-1}[i]}
            \cap
                \smashoperator{\bigcap_{i=r+1}^h}
                Z_{\sigma,\tau^{-1}[i]}
        \right)
    \end{equation}
    where $\CS{r}$ denotes the set of all canonicalisation traces of
    order $r$ on $h$ units and
    \begin{align*}
        X^\delta_I
            &= \Set{(\abcrange, d)\in\W_h}
                {\begin{gathered}\textstyle
                    \forall i \in I,
                    b_i = 0
                \text{ and}\\\textstyle
                    d+\sum_{i \in I} a_i \tanh(c_i) = \delta
                \end{gathered}}
            ;
    \\  Y^{\alpha, \beta, \gamma}_{\sigma, I}
            &= \Set{(\abcrange, d)\in\W_h}
                {\begin{gathered}\textstyle
                    \forall i \in I,
                    \sigma_i \cdot (b_i, c_i) = (\beta, \gamma)
                \\\text{and }\textstyle
                    \sum_{i \in I} \sigma_i a_i = \alpha
                \end{gathered}}
            ;\text{ and}
    \\  Z_{\sigma,I}
            &= \Set{(\abcrange, d)\in\W_h}
                {\begin{gathered}\textstyle
                    \forall i, j \in I,
                    \sigma_i\cdot(b_i, c_i) = \sigma_j\cdot(b_j, c_j)
                \\\text{and }\textstyle
                    \sum_{i \in I} \sigma_i a_i = 0
                \end{gathered}}
            .
    \end{align*}
\end{theorem}

\begin{proof}
    Suppose $w' = (\prange{a'}{b'}{c'}, d) \in \W_h$ is in the union in
        (\ref{eq:characterisation}),
    and therefore in the intersection for some canonicalisation trace
        $(\sigma, \tau) \in \CS{r}$.
    Then $f_{w'} = f_v = f_w$, as follows:
    \begin{align*}
        f_{w'}(x)
        & = d'
            + \smashoperator{\sum_{i\in\tau^{-1}[0]}}
                a'_i \tanh(b'_i x + c'_i)
            + \sum_{j = 1}^r
                \smashoperator{\sum_{\hspace{2.0em}i\in\tau^{-1}[j]}}
                    a'_i \tanh(b'_i x + c'_i)
            + \smashoperator{\sum_{j=r+1}^h}
                \smashoperator{\sum_{\hspace{2.5em}i\in\tau^{-1}[j]}}
                    a'_i \tanh(b'_i x + c'_i)
        \\
        & = \delta + \sum_{j=1}^r \alpha_j \tanh(\beta_j x_i + \gamma_j)
            \text{ since }w' \in
                    X^\delta_{\tau^{-1}[0]}
                    \cap
                        \smashoperator{\bigcap_{j=1}^r}
                        Y^{\alpha_j,\beta_j,\gamma_j}_{\sigma,\tau^{-1}[j]}
                    \cap
                        \smashoperator{\bigcap_{j=r+1}^h}
                        Z_{\sigma,\tau^{-1}[j]}
                    .
    \end{align*}

    Now, suppose $w' \in \fneqclass{w}$.
    Construct a canonicalisation trace $(\sigma, \tau) \in \CS{r}$ following
    the execution of \cref{algo:parameter-canonicalisation} on $w'$.
    Set $\sigma_i = -1$ where $\sign{b'_i} = -1$, otherwise $+1$.
    Construct $\tau$ from identity as follows.
        In each Stage 1 iteration, if the second branch is chosen,
        remap $\tau(i)$ to $0$.
        If the third or fourth branch is chosen, for $k \in \tau^{-1}[i]$
        (including $i$ itself), remap $\tau(k)$ to $j$.
        Finally, incorporate the Stage 3 permutation $\pi$:
            simultaneously for $k \notin \tau^{-1}[0]$,
                remap $\tau(k)$ to $\pi(\tau(k))$.
    
    Note $\CALL{Canonicalise}(w') = v$ by
        \cref{thm:parameter-canonicalisation}.
    Then
        $w' \in \smash{X_{\tau^{-1}[0]}^\delta}$
    because $\tau^{-1}[0]$ contains 
    exactly those units incorporated into $\delta$.
    Moreover, for $j = \range1r$,
        $w' \in Y_{\sigma,\tau^{-1}[j]}^{\alpha_j,\beta_j,\gamma_j}$,
    because $\tau^{-1}[j]$ contains exactly those units incorporated into
    unit $j$ of $v$, and $\sigma$ their relative signs ($\beta_j > 0$).
    Likewise, for $j \in \range{r+1}{h}$,
        $w' \in Z_{\sigma,\tau^{-1}[j]}$
    (which is vacuous if $\tau^{-1}[j]$ is empty).
\end{proof}

\begin{remark}
    If $w \in \W_h$ is irreducible, then $\rank(w) = h$.
    For $(\sigma, \tau) \in \CS{h}$, $\tau$ is a permutation (since the range
    must include $\setrange1h$).
    The set of traces therefore corresponds to the set of transformations
    generated by unit negations and transpositions, as in
        \citet{Sussmann1992}.
\end{remark}

\begin{remark}
    \label{remark:fukumizu-fneq}
    When $\rank(w) = h-1$, there are, modulo sign vectors and permutations,
    essentially three canonicalisation traces, corresponding to the three
    ways of adding an additional unit to a $(h-1)$-unit network discussed
    by \citet{Fukumizu+Amari2000} and \citet{Fukumizu+2019}:
        to introduce a new constant unit or one with zero output, or 
        to split an existing unit in two.
\end{remark}

\begin{remark}
    \label{remark:simsek-fneq}
    Similarly, in \citet[Definitions 3.2 and 3.3]{Simsek+2021},
    an $(r+j)$-tuple coupled with a permutation play the role of $\tau$ in
    characterising the \emph{expansion manifold}, akin to the
    functional equivalence class but from the dual perspective of adding units
    to an irreducible parameter.
    \citet{Simsek+2021} study a setting without a unit negation symmetry,
    so there is no need for a sign vector.
\end{remark}

\section{Path Connectivity}
\label{sec:connectivity}

In this section, we show that the reducible functional equivalence class
is piecewise linear path-connected (\cref{thm:connectivity}),
and, for parameters with rank at most half of the available number of
hidden units, has diameter at most 7 linear segments (\cref{thm:diameter}).

\begin{theorem}
    \label{thm:connectivity}
    Let $w \in \W_h$. If $w$ is reducible, then $\fneqclass{w}$ is piecewise
    linear path-connected.
\end{theorem}

\begin{proof}
    It suffices to show that each reducible parameter
        $w \in \W_h$
    is piecewise linear path-connected in $\fneqclass{w}$ to its canonical
    representative $\CALL{Canonicalise}(w)$.
    The path construction proceeds by tracing the parameter's mutations in
    the course of execution of
        \cref{algo:parameter-canonicalisation}.
    For each iteration of the loops in Stages 1 and 2, and for each
    transposition in the permutation in Stage 3, we construct a multi-segment
    sub-path.
    To describe these sub-paths, we denote the parameter at the beginning of
    each sub-path as $w = (\abcrange,d)$, noting that this parameter is mutated
    throughout the algorithm, but is functionally equivalent to the original
    $w$ at all of these intermediate points.

    \begin{enumerate}[label={\arabic*.}]
        \item
            In each iteration of the Stage 1 loop, the construction depends
            on the chosen branch, as follows.
            Some examples are illustrated in \cref{fig:moves}.
            \begin{enumerate}[label={(\roman*)}]
                \item A direct path interpolating $b_i$ and $c_i$ to zero.
                \item A two-segment path, interpolating $a_i$ to zero and $d$
                    to $d + a_i\tanh(c_i)$, then $c_i$ to zero.
                \item A two-segment path, interpolating $a_i$ to zero and
                    $a_j$ to $a_j+a_i$, then $b_i$ and $c_i$ to zero.
                \item A two-segment path, interpolating $a_i$ to zero and
                    $a_j$ to $a_j-a_i$, then $b_i$ and $c_i$ to zero.
            \end{enumerate}
    \end{enumerate}
    \begin{figure}[!h]
        \centering
        \includegraphics[width=\textwidth]{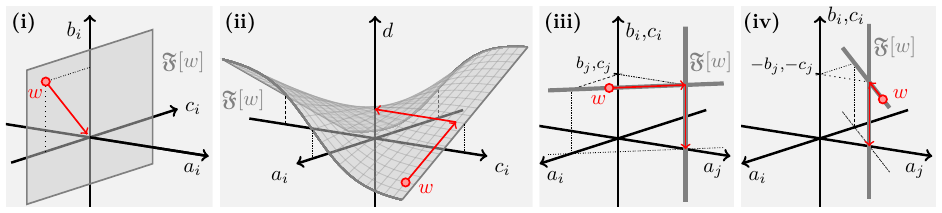}
        \caption{\label{fig:moves}%
        Example paths constructed for each of the Stage~1 branches.
        Other dimensions held fixed.
        }
    \end{figure}
    Since (the original) $w$ is reducible, (the current) $w$ must have gone
    through at least one iteration in Stage 1, and must have at least one
    \emph{blank} unit $k$ with $a_k,b_k,c_k=0$.
    From any such parameter $w$, there is a three-segment path in
    $\fneqclass{w}$ that implements a \emph{blank-exchange manoeuvre}
    transferring the weights of another unit $i$ to unit $k$, and leaving
    $a_i,b_i,c_i=0$:
        first interpolate $b_k$ to $b_i$ and $c_k$ to $c_i$;
        then interpolate $a_k$ to $a_i$ and $a_i$ to zero;
        then interpolate $b_i$ and $c_i$ to zero.
    Likewise, there is a three-segment path that implements a \emph{negative
    blank-exchange manoeuvre}, negating the weights as they are interpolated
    into the blank unit.
    With these manoeuvres noted, proceed:
    \begin{enumerate}[resume]
        \item
            In each iteration of the Stage 2 loop for which $\sign{b_i}=-1$,
            let $k$ be a blank unit, and construct a six-segment path.
                First, blank-exchange unit $i$ into unit $k$.
            Then, negative blank-exchange unit $k$ into unit $i$.
            The net effect is to negate unit $i$.
        \item
            In Stage 3, construct a path for each segment in a decomposition
            of the permutation $\pi$ as a product of transpositions.
            Consider the transposition $(i, j)$.
            If $i$ or $j$ is blank, simply blank-exchange them.
            If neither is blank, let $k$ be a blank unit. Construct a
            nine-segment path, using three blank-exchange manoeuvres,
            using $k$ as `temporary storage' to implement the transposition:
                first blank-exchange units $i$ and $k$,
                then blank-exchange units $i$ (now blank) and $j$,
                then blank-exchange units $j$ (now blank) and $k$ (containing
                    $i$'s original weights).
    \end{enumerate}
    The resulting parameter is the canonical representative and it can be
    verified that each segment in each sub-path remains in $\fneqclass{w}$
    as required.
\end{proof}

\begin{remark}
    \label{remark:simsek-connectivity}
    \citet[Theorem~B.4]{Simsek+2021} construct similar paths to show the
    connectivity of their expansion manifold
        (cf.~\cref{remark:simsek-fneq}).
    They first connect reduced-form parameters using blank-exchange
    manoeuvres and then show inductively that each unit addition preserves
    connectivity.
\end{remark}

\begin{theorem}
    \label{thm:diameter}
    Let $w \in \W_h$. If $\rank(w) \leq \frac{h}{2}$, then $\fneqclass{w}$
    has diameter at most $7$.
\end{theorem}

\begin{proof}
    Let $w \in \W_h$ with $\rank(w) = r \leq \frac{h}{2}$.
    Let $w' \in \fneqclass{w}$.
    We construct a piecewise linear path from $w$ to $w'$ with $7$~segments.
    By \cref{thm:connectivity}, a path exists via the canonical
    representative parameter $v = \CALL{Canonicalise}(w)$.
    However, this path has excessive length.
    We compress the length to~$7$ by exploiting the following opportunities
    to parallelise segments and `cut corners'.
    These optimisation steps are illustrated in \cref{fig:optimisation}.
    \begin{enumerate}[label=(\alph*)]
        \item
            Let the Stage~1 result from \cref{algo:parameter-canonicalisation}
            for $w$ be denoted $u$.
            Let the Stage~1 result for $w'$ be denoted $u'$.
            Instead of following the unit negation and exchange
            transformations from $u$ to $v$, and then back to $u'$,
            we transform $u$ into $u'$ directly, not (necessarily) via $v$.

        \item
            We connect $w$ to $u$ using two segments, implementing all
            iterations of Stage~1{} in parallel.
            The first segment shifts the outgoing weights from the blank
            units to the non-blank units and the output unit bias.
            The second segment interpolates the blank units' incoming weights
            and biases to zero.
            We apply the same optimisation to connect $w$ and $u'$.

        \item
            We connect $u$ and $u'$ using two blank-exchange manoeuvres
            (6~segments), exploiting the majority of blank units as
            `temporary storage'.
            First, we blank-exchange the non-blank units of $u$ into blank
            units of $u'$, resulting in a parameter $\bar u'$ sharing no
            non-blank units with $u'$.
            Then, we (negative) blank-exchange those weights into the
            appropriate non-blank units of $u'$,
            implementing the unit negation and exchange transformations
            relating $u$, $\bar u'$, and $u'$.

        \item
            The manoeuvres in (b) and (c) begin and/or end by interpolating
            incoming weights and biases of blank units from and/or to zero,
            while the outgoing weights are zero.
            We combine adjacent beginning/end segments together, 
            without (necessarily) passing through zero.
            This results in the required seven-segment path, tracing the
            sequence of parameters
                $w, \w1, \w2, \ldots, \w6, w' \in W_h$.
    \end{enumerate}
    
    \begin{figure}[!ht]
        \centering
        \includegraphics{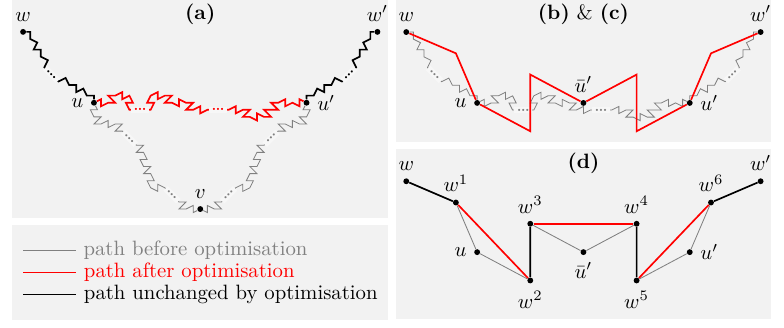}
        \caption{\label{fig:optimisation}%
            A conceptual illustration of the four path optimisations,
            producing a seven-segment piecewise linear path of equivalent
            parameters in a high-dimensional parameter space.
            \textbf{(a)}~Follow unit negation and exchange transformations
                directly between reduced parameters, not via the canonical
                parameter.
            \textbf{(b)}~\& \textbf{(c)}~Parallelise the reduction steps, and
                use the majority of blank units to parallelise the
                transformations.
            \textbf{(d)}~Combine first/last segments of reduction and
                blank-exchange manoeuvres.
        }
    \end{figure}
    
    To describe the constructed path in detail, we introduce the following
    notation for the components of the key parameters
        $w, w', u, u', \w1, \w2, \ldots, \w6 \in \W_h$:
    \begin{align*}
        w   &= (\prange{\pa{w}}{\pb{w}}{\pc{w}}, \pd{w})
    &
        u   &= (\prange{\pa{u}}{\pb{u}}{\pc{u}}, \pd{u})
    \\
        w'  &= (\prange{\pa{w'}}{\pb{w'}}{\pc{w'}}, \pd{w'})
    &
        u'  &= (\prange{\pa{u'}}{\pb{u'}}{\pc{u'}}, \pd{u'})
    \\
        w^k &= (\prange{\pa{k}}{\pb{k}}{\pc{k}}, \pd{k})
    & 
        (k   &= \range16)
    .
    \end{align*}

    Of the $h$ units in $u$, exactly $h-r$ are blank---those in the set $Z$
    from $\CALL{Canonicalise}(w)$.
    Denote the complement set of $r$ non-blank units
        $U = \setrange1h \setminus Z$.
    Likewise, define $Z'$ and $U'$ from $u'$.

    With notation clarified, we can now describe the key points
    $\w1,\ldots,\w6$ in detail, while showing that the entire path is
    contained within the functional equivalence class $\fneqclass{w}$.
    \begin{enumerate}
        \item
            The first segment interpolates 
                each outgoing weight from $\pa{w}_i$ to $\pa{u}_i$,
            and interpolates the output bias from $\pd{w}$ to $\pd{u}$.
            That is, $\w1 = ( \prange {\pa{u}} {\pb{w}} {\pc{w}}, \pd{u})$.
            
            To see that this segment is within $\fneqclass{w}$, observe that
            since the incoming weights and biases are unchanged between the
            two parameters,
                $f_{t \w1 + (1-t) w}(x) = t f_{\w1}(x) + (1-t) f_w(x)$
            for $x \in \Reals$ and $t \in [0,1]$.
            To show that $f_{w} = f_{\w1}$, we construct a function
                $\tau : \setrange1h \to \setrange{0,1}{h}$
            from identity following each iteration of Stage~1 of
                $\CALL{Canonicalise}(w)$:
            when the second branch is chosen, remap $\tau(i)$ to~$0$; and
            when the third or fourth branch is chosen,
            for $k \in \tau^{-1}[i]$ (including $i$~itself), remap $\tau(k)$
            to~$j$.
            Moreover, we define a sign vector $\sigma \in \set{-1,+1}^h$ where
                $\sigma_i = -1$ if $\sign{\pb{w}_i} = -1$, otherwise
                $\sigma_i = +1$.
            Then:
            \begin{align*}
                f_w(x)
                &=  \textstyle
                    \pd{w}
                    + \sum_{j=0}^k
                        \sum_{i\in\tau^{-1}[j]}
                            \pa{w}_i \tanh(\pb{w}_i x + \pc{w}_i)
                \\
                &=  \textstyle
                        \pd{w}
                        + \sum_{i\in\tau^{-1}[0]} \pa{w}_i \tanh(\pc{w}_i)
                    + \sum_{j=1}^h
                        \left(
                            \sum_{i\in\tau^{-1}[j]} \sigma_j \sigma_i \pa{w}_i
                        \right)
                        \tanh(\pb{w}_j x + \pc{w}_j)
                \\
                &=  \textstyle
                    \pd{u} + \sum_{j=1}^h \pa{u}_j \tanh(\pb{w}_j x + \pc{w}_j)
                = f_{\w1}(x)
            .
            \end{align*}

        \item
            The second segment completes the reduction and begins the first
            blank-exchange manoeuvre to store the nonzero units in $Z'$.
            For $i \in U \cap U'$, pick distinct `storage' units
                $j \in Z \cap Z'$.
            There are enough, as $r \leq \frac{h}2$ by assumption thus
                $
                    \cardinality{U \cap U'}
                    = \cardinality{U} - \cardinality{Z \cap U'}
                    = r - \cardinality{Z \cap U'}
                    \leq (h-r) - \cardinality{Z \cap U'}
                    = \cardinality{Z'} - \cardinality{Z \cap U'}
                    = \cardinality{Z' \cap Z}
                $.
            Interpolate unit $j$'s incoming weight
                from $\pb{w}_j$ to $\pb{w}_i$
            and interpolate its bias from $\pc{w}_j$ to $\pc{w}_i$.
            Meanwhile, for all other $j \in Z$, interpolate the incoming
            weight and bias to zero.
            This segment is within $\fneqclass{w}$ as for $j \in Z$,
            $\pa1_j = \pa{u}_j = 0$ by definition of $Z$.
        
        \item
            The third segment shifts the outgoing weights from the units in
                $U \cap U'$
            to the units in $Z \cap Z'$ prepared in step~(2).
            For $i \in U \cap U'$, pick the same storage unit $j$ as in
            step~(2).
            Interpolate
                unit $j$'s outgoing weight from $\pa{u}_j=0$ to $\pa{u}_i$
            and interpolate
                unit $i$'s outgoing weight from $\pa{u}_i$ to zero.
            This segment is within $\fneqclass{w}$ as $\pb2_i = \pb2_j$ and
            $\pc2_i = \pc2_j$ by step~(2).
        
        \item
            The fourth segment completes the first blank-exchange manoeuvre
            and begins the second, to form the units of $u'$.
            For $i \in U'$,
                interpolate unit $i$'s incoming weight
                from $\pb{3}_i$ to $\pb{u'}_i$
            and interpolate its bias
                from $\pc{3}_i$ to $\pc{u'}_i$.
            This segment is within $\fneqclass{w}$ because
            for $i \in U' \cap Z$, $\pa{3}_i = \pa{u}_i = 0$ by definition of
                $Z$,
            and for $i \in U' \cap U$, $\pa{3}_i = 0$ by step~(3).
        
        \item
            The fifth segment shifts the outgoing weights from the selected
            units in $Z'$ to the units in $U'$ prepared in step~(4).
            We simply interpolate each unit $i$'s outgoing weight to
            $\pa{u'}_i$.
            
            To see that the segment is within $\fneqclass{w}$, note that
            $u$ and $u'$ are related by some unit negation and exchange
            transformations.
            Therefore, there is a correspondence between their sets of
            nonzero units, such that corresponding units have the same
            (or negated) incoming weights and biases.
            Due to steps (2)--(4) there are $r$ `storage' units in $\w4$ with
            the weights of the units of $u$, and the correspondence extends
            to these storage units.
            Since the storage units are disjoint with $U'$, this fifth segment
            has the effect of interpolating the outgoing weight of each of the
            storage units $j \in Z'$ in $\w4$
                from $\pa{u}_i$ to zero (where $i$ is as in step~(3)),
            while interpolating the outgoing weight of its corresponding
            unit $k \in U'$
                from zero to $\pm \pa{u}_i = \pa{u'}_k$
            (where the sign depends on the unit negation transformations
            relating $u$ and $u'$).

        \item
            The sixth segment completes the second blank-exchange manoeuvre
            and begins to reverse the reduction.
            For $i \in Z'$,
                interpolate unit $i$'s incoming weight from $\pb5_i$ to
                $\pb{w'}_i$,
            and interpolate its bias from $\pc5_i$ to $\pc{w'}_i$.
            This segment is within $\fneqclass{w}$ as
                for $i \in Z'$,
                $\pa{5}_i = \pa{u'}_i = 0$ by definition of $Z'$.

        \item
            The seventh segment, of course, interpolates from $\w6$ to $w'$.
            To see that this segment is within $\fneqclass{w}$, note that
            by steps~(5) and~(6),
                $\w6 = (\prange{\pa{u'}}{\pb{w'}}{\pc{w'}}, \pd{u'})$
            (noting $\pd{u} = \pd{u'}$ since the output unit's bias is
            preserved by unit transformations).
            So the situation is the reverse of step~(1), and a similar proof
            applies.
            \qedhere
    \end{enumerate}
\end{proof}

\section{Discussion}
\label{sec:discussion}

In this paper, we have investigated the functional equivalence class for
reducible neural network parameters, and its connectivity properties.
These reducible functional equivalence classes are a complex union of
manifolds, displaying the following rich qualitative structure.
\begin{itemize}
    \item
        There is a central discrete array of reduced-form parameters,
        with a maximal number of blank units spread throughout an
        irreducible subnetwork.
        These reduced-form parameters are related by unit negation and
        exchange transformations, like for irreducible parameters.
    
    \item
        Unlike in the irreducible case, these reduced-form parameters are
        connected by a network of piecewise linear paths.
            Namely, these are (negative) blank-exchange manoeuvres, and,
            when there are multiple blank units, simultaneous parallel
            blank-exchange manoeuvres.

    \item
        Various manifolds branch away from this central network, tracing
        in reverse the various reduction operations (optionally in parallel).
        Dually, these manifolds trace methods for \emph{adding} units
            \citep[cf.,][]{Fukumizu+Amari2000,Fukumizu+2019,Simsek+2021}.
\end{itemize}

\Cref{thm:diameter} establishes that with a majority of blank units, the
diameter of this parameter network becomes a small constant number of linear
segments.
With fewer blank units it will sometimes require more blank-exchange
manoeuvres to traverse the central network.
Future work could investigate the trade-offs between shortest path length and
rank for different unit permutations.

\paragraph{Towards modern architectures.}

We have studied single-hidden-layer hyperbolic tangent networks, but
structural redundancies arising from zero, constant, or proportional units
(reducibility conditions \ref{item:reducibility:1}--\ref{item:reducibility:3})
are a generic feature of feed-forward network components.
Unit negation symmetries are characteristic of odd nonlinearities;
other nonlinearities will exhibit similar redundancies due to their own
affine symmetries.
In more complex architectures there will be additional sources of
redundancy, such as interactions between layers or specialised computational
structures.

We call for future work to seek out, catalogue, and thoroughly investigate
such sources of redundancy, rather than assuming their irrelevance as part
of measure zero subset of the parameter space.
Our results serve as a starting point for future work in this direction.
The results of \citet{Vlacic+Bolcskei2021}, significantly generalising
\citet{Sussmann1992}, would be a useful complement.

\paragraph{Functional equivalence and deep learning.}

Functionally equivalent parameters have equal loss.
Continuous directions and piecewise linear paths within reducible functional
equivalence classes
    (Theorems~\ref{thm:characterisation}, \ref{thm:connectivity}, and
    \ref{thm:diameter})
therefore imply flat directions and equal-loss paths in the loss landscape.
More broadly, the set of low- or zero-loss parameters is a union of
functional equivalence classes, including, possibly (or \emph{necessarily},
given sufficient overparameterisation), reducible ones.

Understanding reducible functional equivalence classes may be key to
understanding these topics.
Of special interest is the connection to theoretical work involving
    unit pruning \citep{Kuditipudi+2019}
and permutation symmetries \citep{Brea+2019}.
Of course, having the same loss does not imply functional equivalence---indeed,
\citet{Garipov+2018} observe functional non-equivalence in low-loss paths.
The exact relevance of reducible parameters to these topics remains to be
clarified.

If the loss landscape is smooth, the comments above hold approximately for
irreducible parameters that are merely near some reducible parameter.
Future work should develop techniques to measure proximity to low-rank
parameters\anon[,]{ \citep[see][]{me-thesis,me-prank},}
and empirically investigate the prevalence of approximate reducibility among
parameters encountered during learning.

\section{Conclusion}
\label{sec:conclusion}

While reducible parameters comprise a measure zero subset of the parameter
space, their functional equivalence classes may still be key to understanding
the structure of the parameter space and, in turn, the loss landscape on which
deep learning takes place.
We have taken the first step towards understanding functional equivalence
beyond irreducible parameters, by investigating the setting of
single-hidden-layer hyperbolic tangent networks.
Due to structural redundancy, reducible functional equivalence classes are
much richer than their irreducible counterparts.
By accounting for various kinds of structural redundancy,
we offer a characterisation of reducible functional equivalence classes
and an investigation of their piecewise linear connectivity properties.


\anon{
\section*{Acknowledgements}
\addcontentsline{toc}{section}{Acknowledgements}

Contributions (\ref{itm:contrib:1}), (\ref{itm:contrib:2}), and
(\ref{itm:contrib:3}) also appear in MFR's minor thesis
    \citep[\S5]{me-thesis}.
MFR received financial support from the Melbourne School of Engineering
Foundation Scholarship and the Long-Term Future Fund while completing
this research.
We thank Daniel Murfet for providing helpful feedback during this research
and during the preparation of this manuscript.
}


\bibliographystyle{far}
\bibliography{main}
\addcontentsline{toc}{section}{References}

\clearpage
\appendix

\section{Generalising to multi-dimensional inputs and outputs}
\label{apx:multidim}

In this appendix, we consider a slightly more general family of architectures
than that introduced in \cref{sec:prelims}.
Namely, we consider a family of fully-connected, feed-forward neural network
architectures with
    $n \in \PosNats$ input units,
    $m \in \PosNats$ biased linear output units,
and a single hidden layer of
    $h \in \Natz$ biased hidden units
    with the hyperbolic tangent nonlinearity.
With minor modifications, described in the remainder of this appendix, all
definitions, algorithms, theorems, and proofs directly generalise from the
case $n=m=1$ to arbitrary $n$ and $m$.

\paragraph{Multi-dimensional architecture.}

Let $n \in \PosNats$,
    $m \in \PosNats$,
and $h \in \Natz$.
Define the generalised parameter space
    $\W[n,m]_h = \Reals^{(n+m+1)h+m}$.
The weights and biases of the network's units are encoded in the parameter
vector in the format
    $(\abcrange, d) = w \in \W[n,m]_h$
where
    for each hidden unit $i = \range1h$ there is
        an \emph{outgoing weight vector} $a_i \in \Reals^m$,
        an \emph{incoming weight vector} $b_i \in \Reals^n$,
        and a \emph{bias} $c_i \in \Reals$;
    and $d \in \Reals^m$ is
        an \emph{output unit bias vector}
        containing one bias value for each output unit.
This time, $w$ indexes a multi-dimensional mathematical function
    $f_w : \Reals^n \to \Reals^m$
defined as follows:
\begin{equation}
    \label{eq:fwnm}
    f_w(x) = d + \sum_{i=1}^h a_i \tanh(b_i \cdot x + c_i).
\end{equation}

Note that we use the same tuple notation and ordering
    $(\abcrange, d)$
but now the $a_i$, the $b_i$, and $d$ all denote multi-component vectors.
Accordingly, in \cref{eq:fwnm},
    $b_i$ and $x$ are now multiplied using the inner (dot) product, rather
    than scalar multiplication, since they are both vectors in $\Reals^n$.
Moreover,
    $a_i \in \Reals^m$ as a vector is to be multiplied by the scalar
    $\tanh(b_i \cdot x + c_i)$. That is, the sum is over vectors of
    contributions to output units from each hidden unit.

To generalise the results of the main paper to this setting the first change
necessary is to replace all mentions of scalar weights with these vectors of
weights, and other similar changes such as reading the literal zero as vector
zero where appropriate.

\paragraph{Signing and sorting incoming weight vectors.}

The lexicographic order on $\Reals^n$, denoted $\xleq$, is a relation such
that for $u, v \in \Reals^n$, $u \xleq v$ if and only if $u=v$ or, in the
first index $i=\range1n$ where $u$ and $v$ differ, $u_i < v_i$.
From this definition we follow the usual conventions in defining
    $\xlt$, $\xgt$, and $\xgeq$.
Finally, define the \emph{lexicographic sign} of $v \in \Reals^n$, denoted
$\xsign{v}$, as follows:
\begin{equation*}
    \xsign{v}
    = \begin{cases}
        +1  &\quad (v \xgt 0), \\
        ~0  &\quad (v  =   0), \\
        -1  &\quad (v \xlt 0).
    \end{cases}
\end{equation*}

The parameter canonicalisation algorithm and some of the other theorems and
proofs make repeated use of the signs of incoming weight vectors.
The lexicographic sign satisfies the requisite properties of the scalar sign
function in these uses and so the second change necessary to generalising the
results is to replace uses of $\sign{\cdot}$ with uses of $\xsign{\cdot}$.

This lexicographic order relation is of course also a total order \citep[see,
e.g.,][Theorem 4.1.11]{Harzheim2005}.
Therefore, it allows one to sort a list of vectors.
Sorting units by decreasing incoming weights is a key step in Stage~3 of
    \cref{algo:parameter-canonicalisation},
and so the third change necessary is to use decreasing lexicographic order
($\xgeq$) in this stage.

\paragraph{Generalising Sussmann's equivalence theorem.}
        
The proofs in the main paper rely on the results of \citet{Sussmann1992} on
the equivalence between reducibility and non-minimality, and the fact that
irreducible functionally equivalent parameters are related by unit negation
and exchange transformations.
\citet{Sussmann1992} studied a setting with multiple input units but only
a single output unit.
\Cref{lemma:generalised-sussmann-1,lemma:generalised-sussmann-2}
    generalise these results to the multi-output setting.\footnote{%
        The proofs reduce the multi-output case to the single-output case, so
        they still rely on the results of \citet{Sussmann1992}.
        A generalisation similar to \cref{lemma:generalised-sussmann-1} is
        given by \citet{Fukumizu1996}.
}
The final necessary change to generalise the results in the main paper is to
replace all references to Sussmann's results with references to
    \cref{lemma:generalised-sussmann-1} or
    \cref{lemma:generalised-sussmann-2}.

The definitions of unit negation and exchange transformations, reducibility,
and non-minimality all generalise to arbitrary $n$ and $m$ with the
above-mentioned changes. These definitions are repeated here for convenience.

A \emph{unit negation transformation} is a function
    $\sigma_i : \W[n,m]_h \to \W[n.m]_h$ for $i = \range1h$,
where
\begin{equation*}
    \sigma_i(\abcrange, d)
    = (
        a_1,b_1,c_1,
        \ldots,
        {\color{accent2}-}a_i,
        {\color{accent2}-}b_i,
        {\color{accent2}-}c_i,
        \ldots,
        a_h,b_h,c_h,
        d
    ).
\end{equation*}

A \emph{unit exchange transformation} is a function
    $\tau_{i,j} : \W[n,m]_h \to \W[n,m]_h$
for $i,j=\range1h$, where
\begin{align*}
    \tau_{i,j}(\abcrange, d)
    &= (
        a_1, b_1, c_1,
        \ldots,
        c_{i-1},
        a_{\color{accent2}j},
        b_{\color{accent2}j},
        c_{\color{accent2}j},
        a_{i+1},
    \\ & \hspace{5.6em} 
        \ldots,
        c_{j-1},
        a_{\color{accent2}i},
        b_{\color{accent2}i},
        c_{\color{accent2}i},
        a_{j+1},
        \ldots,
        a_h,b_h,c_h,
        d
    ).
\end{align*}

A parameter $w = (\abcrange, d) \in \W[n,m]_h$ is \emph{reducible} if and
only if it satisfies any of the following conditions
    (otherwise, $w$ is \emph{irreducible}):
\begin{enumerate}[label=(\roman*)]
    \item
        $a_i = 0$ for some $i$,
    \item
        $b_i = 0$ for some $i$,
    \item
        $(b_i, c_i) = (b_j, c_j)$ for some $i \neq j$, or
    \item
        $(b_i, c_i) = (-b_j, -c_j)$ for some $i \neq j$.
\end{enumerate}

A parameter $w \in \W[n,m]_h$ is \emph{non-minimal} if and only if $w$ is
functionally equivalent to some $w' \in \W[n,m]_{h'}$ with fewer hidden units
$h' < h$.

\begin{lemma}
    \label{lemma:generalised-sussmann-1}
    For $w \in \W[n,m]_h$,
    $w$ is reducible if and only if $w$ is non-minimal.
\end{lemma}

\begin{proof}
($\Rightarrow$):
    A smaller functionally equivalent parameter can be constructed as follows.
    \begin{enumerate}[label=(\roman*)]
        \item
            If $a_i = 0$ for some $i$, then hidden unit $i$ fails to
            contribute to the function.
            Construct a functionally equivalent parameter
                $w' \in \W[n,m]_{h-1}$
            with hidden unit $i$ omitted:
            \begin{equation*}
            w' = (
                a_1, b_1, c_1,
                \ldots,
                a_{i-1}, b_{i-1}, c_{i-1},
                a_{i+1}, b_{i+1}, c_{i+1},
                \ldots,
                a_h, b_h, c_h,
                d
            ).
            \end{equation*}

        \item
            If $b_i = 0$ for some $i$, then hidden unit $i$ contributes
            only a constant to the function.
            Construct a functionally equivalent parameter
                $w' \in \W[n,m]_{h-1}$
            with hidden unit $i$ omitted and the output unit bias vector
            changed to compensate:
            \begin{equation*}
            w' = (
                a_1, b_1, c_1,
                \ldots,
                a_{i-1}, b_{i-1}, c_{i-1},
                a_{i+1}, b_{i+1}, c_{i+1},
                \ldots,
                a_h, b_h, c_h,
                d + a_i \tanh(c_i)
            ).
            \end{equation*}

        \item
            If $(b_i, c_i) = (b_j, c_j)$ for some $i \neq j$,
            then hidden units $i$ and $j$ contribute proportionately.
            They can be combined into a single unit (say,~$j$) with the same
            incoming weights and bias, and a combined outgoing weight vector.
            Construct a functionally equivalent parameter
                $w' \in \W[n,m]_{h-1}$ accordingly:
            \begin{equation*}
                w' = (
                    a_1, b_1, c_1,
                    \ldots,
                    c_{i-1}, a_{i+1},
                    \ldots,
                    c_{j-1}, a_j + a_i, b_j, c_j, a_{j+1},
                    \ldots,
                    a_h, b_h, c_h,
                    d
                ).
            \end{equation*}

        \item
            If $(b_i, c_i) = - (b_j, c_j)$ for some $i \neq j$,
            then hidden units $i$ and $j$ contribute in negative proportion.
            Due to the odd property of $\tanh$ they can be combined into a
            single unit (say,~$j$) with incoming weight and bias vectors
            $(b_j, c_j)$ and a combined outgoing weight vector.
            Construct a new parameter $w' \in \W[n,m]_{h-1}$ accordingly:
            \begin{equation*}
                w' = (
                    a_1, b_1, c_1,
                    \ldots,
                    c_{i-1}, a_{i+1},
                    \ldots,
                    c_{j-1}, a_j - a_i, b_j, c_j, a_{j+1},
                    \ldots,
                    a_h, b_h, c_h,
                    d
                ).
            \end{equation*}

    \end{enumerate}
    In all cases, the new parameter $w' \in \W[n,m]_{h-1}$ has $f_{w'} = f_w$,
    so $w$ is non-minimal.

($\Leftarrow$):
    We reduce to the single-output case and apply the result of
    \citet{Sussmann1992} to show that $w$ satisfies at least one of the
    reducibility conditions.

    To reduce to the single-output case, we introduce some notation.
    From the function
        $f_w : \Reals^n \to \Reals^m$
    define a series of component functions
        $f_w^{(1)}, f_w^{(2)}, \ldots, f_w^{(m)} : \Reals^n \to \Reals$
    such that for $x \in \Reals^n$,
    \begin{equation*}
        f_w(x) = \left(f_w^{(1)}(x), f_w^{(2)}(x), \ldots, f_w^{(m)}(x)\right).
    \end{equation*}
    Each of these component functions is a simple neural network function
    in an architecture with $n$~input units and~$1$ output unit,
    corresponding to a subgraph of the connection graph of the original
    neural network, as illustrated in \cref{fig:nn-subgraphs}.
    
    \begin{figure}[!h]
        \centering
        \includegraphics{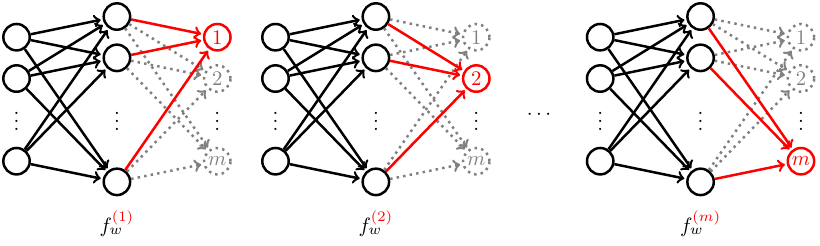}
        \caption{%
            \label{fig:nn-subgraphs}
            The connection graphs of the component functions of $f_w$.
            Included units and weights are solid.
            The hidden units of each network share the same incoming weights
            (and biases, not shown).
        }
    \end{figure}

    Denote the corresponding (overlapping) subvectors of $w \in \W[n,m]_h$
    as
        $w_{(1)}, \ldots, w_{(m)} \in \W[n,1]_h$.
    That is, for $\mu = \range1m$,
    \begin{equation*}
        w_{(\mu)}
        = (
            a_{1,\mu}, b_1, c_1,
            \ldots,
            a_{h,\mu}, b_h, c_h,
            d_\mu
        )
        \in \W[n,1]_h
        .
    \end{equation*}
   
    Now, let $w' = (\prange{a'}{b'}{c'}, d') \in \W_{h'}$
    such that $f_{w'} = f_w$ where $h'$ is the smallest number of hidden
    units required to implement $f_w$ ($h' < h$ by assumption of
    non-minimality).
    Apply the same decomposition to $f_{w'}$ to define
        $f_{w'}^{(1)}, \ldots, f_{w'}^{(m)}$,
    and to define
        $w'_{(1)}, \ldots, w'_{(m)} \in \W[n,1]_{h'}$.
    
    Apply the results of \citet{Sussmann1992} as follows.
    Since $f_w = f_{w'}$,
        $f_{w}^{(\mu)} = f_{w'}^{(\mu)}$ for $\mu = \range1m$.
    It follows that for each $w_{(\mu)}$, $w'_{(\mu)}$ is a functionally
    equivalent parameter using fewer units.
    Therefore, the reducibility conditions (in the special case of $m=1$)
    must hold for each $w_{(\mu)}$ \citep{Sussmann1992}.

    Since conditions (ii--iv) only depend on incoming weights and biases,
    if any of these conditions hold for any $w_{(\mu)}$, then they must
    also hold for $w$ itself (which shares the same incoming weights and
    biases), and the proof is complete.
    It remains only to consider the case in which conditions (ii--iv) fail to
    hold for any $w_{(\mu)}$, and to show that condition (i) holds for $w$
    itself in this case.

    We must introduce yet further notation.
    For $i = \range1h$ denote by $\varphi_i : \Reals^n \to \Reals$
    the function $\varphi_i(x) = \tanh(b_ix+c_i)$.
    Similarly for $j = \range1{h'}$ denote by $\psi_j : \Reals^n \to \Reals$
    the function $\psi_j(x) = \tanh(b'_jx+c'_j)$.
    Then, since we have ruled out reducibility conditions (ii--iv) for $w$,
    no $\varphi_i$ is constant (ii) and no two are proportional (iii, iv).
    The same holds for the $\psi_j$---conditions (i--iv) do not hold for
    $w'_{(\mu)}$, since $h'$ was assumed to be minimal.
    Yet, for $\mu=\range1m$, the linear combination of functions
        \begin{equation*}
            d_\mu + \sum_{i=1}^h a_{i,\mu} \varphi_i
              - d'_\mu - \sum_{j=1}^{h'} a'_{j,\mu} \psi_j
            = f_w^{(\mu)} - f_{w'}^{(\mu)}
            = 0
        \end{equation*}
    yields the zero function.
    This linear combination remains when excluding those terms with 
        $a_{i,\mu} = 0$ or $a'_{j,\mu} = 0$.
    Applying the same reasoning as that in \citet{Sussmann1992}, due to
    the independence property of the hyperbolic tangent function
        \citep[Lemma 3.1]{Sussmann1992}
    the remaining terms must be in bijection, such that
    \begin{equation}\label{eq:unit-bondage}
        \varphi_i = \pm \psi_j
    \end{equation}
    for some $j$ with $a'_{j,\mu} \neq 0$
    for each $i$ with $a _{i,\mu} \neq 0$.
    
    To complete the proof, note that these relationships
        (\ref{eq:unit-bondage})
    between the units of $w$ and $w'$ are independent of $\mu$.
    However, the relationships are ``exclusive'' in the sense that no two
        $\varphi_i$ can be proportional to the same $\psi_j$,
    else they would also be proportional to each other (ruled out above).
    Since there are only $h'$ units $\subrange1{h'}\psi$, it follows that
    there must be one hidden unit $i$ (actually at least $h-h'$ many units)
    for which $a_{i,\mu} = 0$ for all $\mu=\range1m$
        (allowing $\varphi_i$ to avoid any such relationship).
    That is, $a_i = \vecrange{i,1}{i,m}a = 0$, satisfying condition (i) for
    $w$ as required.
\end{proof}

\begin{lemma}
    \label{lemma:generalised-sussmann-2}
    Let $w \in \W[n,m]_h$ be irreducible, and let $w' \in \W[n,m]_h$.
    If $w$ and $w'$ are functionally equivalent then there exists a
    compositional chain of unit negation and exchange transformations,
    collectively a transformation
        $T : \W[n,m]_h \to \W[n,m]_h$,
    such that
        $w' = T(w)$.
\end{lemma}

\begin{proof}
    Once again, we reduce to the case $m=1$ and appeal to \citet{Sussmann1992}.

    Suppose $w' \in \fneqclass{w}$. Introduce the same decomposition of
    the two neural networks as in the proof of
        \cref{lemma:generalised-sussmann-1},
    namely, the component functions
        $f_w^{(1)}, \ldots, f_w^{(m)}, f_{w'}^{(1)}, \ldots, f_{w'}^{(m)}$
    implemented by the parameter subvectors
        $w_{(1)}, \ldots, w_{(m)}, w'_{(1)}, \ldots, w'_{(m)} \in \W[n,1]_h$
    (cf.~\cref{fig:nn-subgraphs}).
    
    For $\mu = \range1m$, since $f_w = f_{w'}$, we have that
        $f_w^{(\mu)} = f_{w'}^{(\mu)}$.
    Now, $w_{(\mu)}$ and $w'_{(\mu)}$ are not necessarily irreducible,
    but if they are reducible then it is only by condition~(i), since
        $w_{(\mu)}$ and $w'_{(\mu)}$
    have the incoming weights and biases of
        $w$ and $w'$ respectively
    ($w$ is irreducible by assumption;
    $w'$ is irreducible because, with the same number of units as $w$,
    it is necessarily minimal, and irreducibility follows by
        \cref{lemma:generalised-sussmann-1}).
    Remove such units with zero outgoing weight from 
        $w_{(\mu)}$ and $w'_{(\mu)}$
    to produce new, functionally equivalent irreducible parameters
        $u_{(\mu)}, u'_{(\mu)} \in \W[n,1]_{\rank(w_{(\mu)})}$.
    Now by \citet[Theorem 2.1]{Sussmann1992}
    there exists a chain of unit negation and exchange transformations
        $T_\mu$ such that $u_{(\mu)} = T_\mu(u'_{(\mu)})$.
    
    For each $\mu$, $T_\mu$ implies a relationship between the units of
    $w_{(\mu)}$ and $w'_{(\mu)}$ with nonzero outgoing weights, including
    possible negations and permutations of these units.
    This same relationship must hold between those units of $w$ and $w'$
    since they share incoming weights and biases with $w_{(\mu)}$ and
    $w'_{(\mu)}$, and (since $w$ is irreducible, conditions (ii--iv))
    these incoming weights are nonzero and the incoming weight and bias
    vectors are absolutely distinct between units of the same parameter.
    Moreover, all units are involved in some such relationship because
    no unit of $w$ or $w'$ can have zero outgoing weight vector by
    reducibility condition (i).

    So, one can construct from these implied relationships a composition
    of unit negation and exchange transformations relating $w$ and $w'$
    as required.
\end{proof}

\end{document}